\documentclass[11pt]{article}
\usepackage{amsmath,amssymb,amsthm}
\usepackage{color}
\usepackage{pdfsync,enumitem,dsfont}
\usepackage{geometry}
\usepackage[T1]{fontenc}
\usepackage[latin9]{inputenc}
\usepackage{lmodern}
\usepackage{geometry}
\usepackage{amsthm}
\usepackage{amsmath}
\usepackage{amssymb}
\usepackage{esint}
\usepackage[numbers, sort]{natbib}

\usepackage{subfig}
\usepackage{graphicx}

\newcommand{\E}{\mathbf{E}}
\newcommand{\Prob}{\mathbf{P}}
\newcommand{\hist}{\mathcal{F}_{t}}
\newcommand{\A}{\mathcal{A}}
\newcommand{\Y}{\mathcal{Y}}

\newcommand{\N}{\mathbb{N}_0}
\newcommand{\Ent}{\mathbf{H}}
\newcommand{\I}{\mathbf{I}}
\newcommand{\D}{\mathbf{D}}
\newcommand{\boldGamma}{\mathbf{\Gamma}}

\DeclareMathOperator*{\argmax}{argmax}

\theoremstyle{plain}

\newtheorem{thm}{Theorem}
\newtheorem{lemma}[thm]{Lemma}

\newtheorem{example}{Example}

\usepackage{authblk}

\newtheoremstyle{TheoremNum}
        {\topsep}{\topsep}              %%% space between body and thm
        {\itshape}                      %%% Thm body font
        {}                              %%% Indent amount (empty = no indent)
        {\bfseries}                     %%% Thm head font
        {.}                             %%% Punctuation after thm head
        { }                             %%% Space after thm head
        {\thmname{#1}\thmnote{ \bfseries #3}}%%% Thm head spec
\theoremstyle{TheoremNum}
\newtheorem{thmn}{Theorem}

\date{}

\long\def\comment#1{}
\begin{document}
\title{Time-Sensitive Bandit Learning and \\ Satisficing Thompson Sampling}
\author[1]{Daniel Russo}
\author[2]{David Tse}
\author[3]{Benjamin Van Roy}
\affil[1]{Northwestern University, daniel.russo@kellogg.northwestern.edu}
\affil[2,3]{Stanford University, bvr@stanford.edu}

\maketitle

\begin{abstract}
The literature on bandit learning and regret analysis has focused on contexts where the goal is to converge on an optimal action in a manner that limits exploration costs.  One shortcoming imposed by this orientation is that it does not treat time preference in a coherent manner.  Time preference plays an important role when the optimal action is costly to learn relative to near-optimal actions.  This limitation has not only restricted the relevance of theoretical results but has also influenced the design of algorithms.  Indeed, popular approaches such as Thompson sampling and UCB can fare poorly in such situations. In this paper, we consider discounted rather than cumulative regret, where a discount factor encodes time preference.  We propose {\it satisficing Thompson sampling} -- a variation of Thompson sampling -- and establish a strong discounted regret bound for this new algorithm.
\end{abstract}

\section{Introduction}
As high level motivation, consider a recommendation system that interacts sequentially with a single user. The system faces the classic tradeoff between exploration and exploitation: by experimenting with different recommendations the system can learn to offer more effective personalized recommendations in the future, but this may require some degradation of current performance. While recommendation systems are often used as a motivating example for studying the multi-armed bandit problem, this problem has several salient features that are not addressed well by standard bandit algorithms and analysis (e.g. the UCB1 algorithm and analysis of \citet{auer2002finite}). First, modern recommendation systems have an enormous number of products, but
 when begining to interact with a new user, the system has a great deal of historical data from interactions with different users, and therefore begins with \emph{significant prior knowledge} about the products. This prior knowledge presents itself in multiple ways. As certain products are typically much more popular than others, the system begins with evidence that certain types of recommendations will be more successful than others.  In addition, data can be used to uncover relevant features of items and users, for example through matrix-completion. As a result, experience offering one type of item to a user can  provide significant information about whether they will like a different type of product. Another distinguishing feature of this problem is the presence of a \emph{limited and uncertain horizon}. The \emph{limited} number of interactions means that a recommendation system likely won't have enough experience with each single user to perfectly tailor its recommendations to them. Instead, it is natural to hope for a system that quickly learns to offer highly effective, but still suboptimal, recommendations to its users. The \emph{uncertain} horizon means that one can't predict {\it a priori} how many times the system will interact with a single user. As a result it is especially valuable to have strong performance during early interactions.

This work focuses on developing algorithms and a framework for theoretical analysis to address problems with these salient features. We build on the Thompson sampling algorithm (TS) \citep{thompson1933} and a recent information theoretic analysis of its performance \cite{russo2016info}, but offer substantial advances to this thread of theoretical work.  TS is able to leverage very general forms of prior information, including rich statistical models that encode a relationship between actions, and prior knowledge that some actions are more likely to offer strong performance than others. The information theoretic analysis of TS yields regret bounds that scale with the entropy of the optimal action distribution.  This dependence reflects the performance benefits of prior information but also points to a substantial potential weakness.  In particular, entropy generally increases with the number of actions and becomes infinite when they form a continuum.  Such regret bounds can therefore be irrelevant when action spaces are very large or infinite.  At the heart of this issue is the emphasis Thompson sampling and this information theoretic analysis place on identification of an optimal action.  There are circumstances when a near-optimal action can be identified quickly even though an optimal one proves elusive.

Instead of focusing on cumulative regret, we will compare algorithms based on their expected discounted regret, where the discount factor encodes time preferences. Note that minimizing expected discounted regret is equivalent to minimizing expected undiscounted regret in a problem where the horizon is a geometric random variable, and hence is uncertain. We introduce satisficing Thompson sampling (STS), a modified form of Thompson sampling designed to address problems with limited horizon. We bound discounted-regret by leveraging the information-theoretic concept of rate-distortion, which offers a means for reasoning about the value of information that is useful for identifying near-optimal, not just optimal, actions.  Through simulation results, and instantiating these regret bounds on specific examples, we show STS can dramatically outperform TS and standard UCB algorithms when the optimal action is costly to learn relative to high-performing suboptimal actions.

Many papers \citep{kleinberg2008multi, rusmevichientong2010linearly, bubeck2011xarmed}  have studied bandit problems with continuous action spaces, where it is also necessary to learn only approximately optimal actions. However, because these papers focus on the asymptotic growth rate of regret they implicitly emphasize later stages of learning, where the algorithm has already identified extremely high performing actions but exploration is needed to identify even better actions. Our discounted framework instead focuses on the initial cost of learning to attain good, but not perfect, performance.  Recent papers \cite{francetich2016toolkita,francetich2016toolkitb} study several heuristics for a discounted objective, though without an orientation toward formal regret analysis.  The Knowledge Gradient algorithm of \cite{ryzhov2012knowledge}  also takes time horizon into account and can learn suboptimal actions when its not worthwhile to identify the optimal action. This algorithm tries to directly approximate the optimal Bayesian policy using a one-step lookahead heuristic, but unfortunately there are no performance guarantees for this method. \citet{deshpande2012linear} consider a linear bandit problem with dimension that is too large relative to the desired horizon. They propose an algorithm that limits exploration and learns something useful within this short time frame. \citet{berry1997bandit, wang2009algorithms} and \citet{bonald2013two} study an infinitely-armed bandit problem in which it's impossible to identify an optimal action and propose algorithms to minimizes the asymptotic growth rate of regret. While we will instantiate our general regret bound for STS on the infinitely-armed bandit problem, we use this example mostly to provide a simple analytic illustration. We hope that the flexibility of STS and our analysis framework allow this work to be applied to more complicated time-sensitive learning problems.

\section{Problem Formulation}

An agent sequentially chooses actions $(A_t)_{t\in \N}$ from the action set $\A$ and observes the corresponding outcomes $\left(Y_{t, A_t}\right)_{t\in \N}$. There is a random outcome $Y_{t, a} \in \Y$ associated with each $a\in \A$ and time $t \in \N\equiv \{0,1,2..\}$.  Let $Y_t \equiv (Y_{t,a})_{a \in \A}$ be the vector of outcomes at time $t \in \N$.  There is a random variable $\theta$ such that, conditioned on $\theta$, $(Y_t)_{t \in \N}$ is an iid sequence.  Ours can be thought of as a Bayesian formulation, in which the distribution of $\theta$ represents the agent's prior uncertainty about the true characteristics of the system, and conditioned on $\theta$, the remaining randomness in $Y_t$ represents idiosyncratic noise in observed outcomes.

The agent associates a reward $R(y)$ with each outcome $y\in \Y$.  Let $R_{t,a} \equiv R(Y_{t,a})$ denote the reward corresponding to outcome $Y_{t,a}$.
The history available when selecting action $A_t$ is
$$\hist = (A_0, Y_{0,A_0}, \ldots, A_{t-1}, Y_{t-1,A_{t-1}}).$$
The agent selects actions according to a policy, which is a sequence of functions
$(\pi_t: t \in \N)$, each mapping a history and an exogenous random variable $\xi$ to an action. In particular $A_t= \pi_{t}(\hist, \xi)$ for each $t$, where $\xi$ is some random variable that is independent of $\theta$ and $(Y_{t} : t\in \N)$.

We denote by $Y_{\infty}$ an independent copy of $Y_t$. In particular, $\Prob(Y_{\infty} \in \cdot | \theta)= \Prob(Y_{t} \in \cdot | \theta)$ but conditioned on $\theta$, $Y_{\infty}$ is drawn independently of $(Y_{t}: t \in \N)$. Let
 $A^* \in \,\, \argmax_{a \in \A} \, \E[R(Y_{\infty , a}) | \theta]$
  denote the true optimal action and let $R^*  = \max_{a\in \A}\E[R(Y_{\infty , a}) | \theta]$ denote the corresponding reward.  As a performance metric, we consider
  \emph{expected discounted regret}, defined by
 \[
\E\left[\sum_{t=0}^{\infty} \alpha^{t} (R^* - R_{t,A_t})\right],
\]
which measures a discounted sum of the expected performance gap between a benchmark policy which always chooses the optimal action $A^*$
and the chosen actions $(A_t : t \in \N)$.  This deviates from the typical notion of expected regret in its dependence on a discount factor $\alpha \in [0,1]$.
Regular expected regret corresponds to the case of $\alpha = 1$.  Smaller values of $\alpha$ convey time preference by weighting gaps in nearer-term
performance higher than gaps in longer-term performance.

It is worth noting that minimizing expected discounted regret is equivalent to maximizing expected discounted reward, which is the
objective used in the classical Bayesian formulation of the multi-armed bandit problem \citep{gittins2011multi}.  For problems of the kind we consider,
expected discounted reward can in principle be maximized via dynamic programming.  However, solving the associated dynamic programs
is computationally intractable.  As such, similarly with the bulk of the recent bandit learning literature, we consider heuristic policies
and aim to bound regret, though in this paper we consider a discounted variation of regret.

\section{Algorithms}\label{sec: algorithms}

Thompson sampling (TS) is a popular algorithm that implements a useful decision policy.
Over each $t$th period, TS selects an action $A_t$ as follows:
\begin{enumerate}
\item Sample $\hat{\theta}_t \sim \Prob(\theta | \hist)$
\item Let $A_t \in \arg\max_{a \in \mathcal{A}} \E\left[R_{t,a} | \theta = \hat{\theta}_t\right]$
\end{enumerate}
We will assume that actions are indexed and that ties are broken by selecting the action with the smallest index.
Note that, as should be the case for any decision policy,
we can write TS as $A_t = \pi_t(\hist, \xi)$, for an appropriately defined $(\pi_t: t \in \N)$, where $\xi$ is independent of $\theta$ and $(Y_{t} : t\in \N)$.

As one key contribution of this paper, we introduce a modification of TS, which we will call satisficing Thompson sampling (STS).
While TS aims to identify an optimal action, STS is designed to identify an action that is sufficiently satisfying, or close enough to optimal.
Over each $t$th period, STS selects an action $A_t$ as follows:
\begin{enumerate}
\item Sample $\hat{\theta}_t \sim \Prob(\theta | \hist)$
\item Let $A_t \in \arg\max_{a \in \mathcal{A}} \E\left[R_{t,a} | \theta = \hat{\theta}_t\right]$
\item Let $\hat{\tau} = \min\left\{\tau \in \{1,\ldots,t-1\} : \E\left[R_{t,A_{\tau}} | \theta =\hat{\theta}_t\right] + \epsilon \geq \E\left[R_{t,A_t} | \theta=\hat{\theta}_t\right]\right\}$
\item If $\hat{\tau}$ is not null then let $A_t = A_{\hat{\tau}}$
\end{enumerate}
Note that $\epsilon \geq 0$ is supplied to the algorithm as a tolerance parameter.  When $\epsilon = 0$, STS is equivalent to TS.  Otherwise,
STS attributes preference to selecting previously selected actions.  As we will further discuss and formalize, this can result in substantial benefit
in the face of time preference.  In particular, when the optimal action requires a long time to learn but an $\epsilon$-optimal action can be learned
quickly, STS can quickly achieve $\epsilon$-optimal performance where Thompson sampling would continue to incur significant losses
deploying resources toward eventual identification of the optimal action.

It is worth mentioning that STS can be applied efficiently across the wide variety of problems that are amenable to Thompson sampling.  This includes,
for example, complex parametric bandit problems.  For example, we present in Section \ref{se:computations} computational results involving a
linear bandit problem with many arms and many parameters to learn.

{\bf A probability matching property:} Thompson sampling satisfies a  powerful probability matching property: under TS, $\Prob_{t}(A_t = a) = \Prob_{t}(A^*=a)$ for all $a\in \A$, and so action-sampling probabilities are \emph{matched} to the posterior distribution of the optimal action. Under STS, action-sampling probabilities instead are essentially matched to the posterior-distribution of the first $\epsilon$--optimal action sampled by the algorithm. More precisely, if $\tau = \inf\{t|\, \E[R_{t, A_t}| \theta]  \geq R^*-\epsilon\}$ then at time $t$ STS sets $\Prob_{t}(A_t = A_k) = \Prob_{t}(\tau = k)$ for each $k<t$. With probability $\Prob_{t}(\tau \geq t)$ STS selects a new, or previously un-sampled action. In this way, the algorithm aims to identify a satisfactory action while concentrating exploration effort on the smallest number of arms required to do so.

\section{Example: Infinitely-Armed Deterministic Bandit}\label{sec: infinite deterministic bandit}

To clarify our motivation, we now provide a simple analytic illustration of advantages enjoyed by STS. Consider a problem with a countable action space $\mathcal{A} = \{1,2,\ldots\}$ in which each action $a \in \mathcal{A}$ yields reward $\theta_a$.  Our prior over each $\theta_a$ is independent and uniform over the interval $[0,1]$.  The optimal reward is almost surely $R^* = 1$.

For this problem, which we refer to as the infinitely-armed deterministic bandit problem, TS never selects the same action twice.
This is because, with probability one, no action selected within a finite time horizon yields reward $1$, and as such, at any point in time, there are better
actions than those previously selected by TS. STS, in contrast, stops searching after finding an action that generates reward exceeding $1-\epsilon$.
After such an action is identified, STS repeatedly selects that action.

The benefits of STS can be formalized in terms of bounds on expected discounted regret.  The following result, proved in the appendix, provides an expression for
the expected discounted regret of TS in our infinitely-armed deterministic bandit problem.
\begin{thm}\label{thm: regret of TS}
For all $\alpha \in [0,1]$, under Thompson sampling in the infinitely-armed deterministic bandit problem then
$$\E\left[\sum_{t=0}^\infty \alpha^t (R^* - R_{t,A_t})\right] =\frac{1}{2(1-\alpha)}.$$
\end{thm}
It is enlightening to compare this to the following bound on expected discounted regret of STS in our infinite deterministic bandit problem, which is also proved in the appendix.
\begin{thm}\label{thm: direct regret bound for STS}
For all $\alpha \in [0,1]$, under satisficing Thompson sampling with tolerance $\epsilon = \sqrt{1-\alpha}$ in the infinitely-armed deterministic bandit problem,
$$\E\left[\sum_{t=0}^\infty \alpha^t (R^* - R_{t,A_t})\right] \leq \frac{1}{\sqrt{1-\alpha}}.$$
\end{thm}
For $\alpha$ close to $1$, $1/\sqrt{1-\alpha} \ll 1/(1-\alpha)$, and therefore STS vastly outperforms TS.
In fact, as $\alpha$ approaches $1$, the ratio between expected regret of TS and that of STS goes to infinity.

\section{Computational Examples}
\label{se:computations}

Computational studies involving a broader range of bandit problems further illustrate potential benefits afforded by STS.
In this section, we present results from experiments with four bandit problems.  Each case is designed so that near-optimal
actions can be identified far sooner than the optimal action.  In each case, the per-period regret of STS diminishes more rapidly
than that of TS over early time periods.

Our first is a deterministic bandit problem with 250 actions.  The mean reward
associated with each action is independently sampled from $\text{unif}([0,1])$.  When an action is sampled the realized reward is equal to
the mean reward; in other words, there is no observation noise.  Figure \ref{fig:computational-results}(a) plots per-period regret of TS and STS
over 500 time periods, averaged over 5000 simulations, each with an independently sampled problem instance.  For STS, we used a tolerance
parameter of $0.05$.

We next consider a problem that is the same as our previous one except with observation noise.  In particular, instead of observing
the mean reward, after selecting an action, we observe a binary reward that is one with probability equal to the mean reward.
Figure \ref{fig:computational-results}(b) plots average per-period regret over 5000 simulations.  For STS, we used a tolerance
parameter of $0.05$.

We now consider another bandit problem with 250 actions, each with a mean reward sampled independently independently from $N(0,1)$.
Upon taking an action, we observe the sum of the action's mean reward and an independent $N(0,1)$ sample that represents observation noise.
Figure \ref{fig:computational-results}(c) plots per-period regret of TS and STS
over 500 time periods, averaged over 5000 simulations, each with an independently sampled problem instance.  For STS, we used a tolerance
parameter of $0.5$.

Finally, we consider a bandit problem with mean rewards given by a linear function.  In particular, mean rewards are given
by a vector $L \theta \in \Re^{|\mathcal{A}|}$, where $L \in \Re^{|\mathcal{A}| \times M}$ is a randomly generated loadings matrix,
with each row independently drawn uniformly from the unit sphere, and $\theta \in \Re^M$ is sampled from $N(0,I)$.
For our computational study, we let $\mathcal{A} = \{1, \ldots, 250\}$ and $M=250$.  The decision-maker knows $L$
and begins with a $N(0,I)$ prior on $\theta$.
Upon taking an action, we observe the sum of the action's mean reward and an independent $N(0,2)$ sample that represents observation noise.
Figure \ref{fig:computational-results}(d) plots per-period regret of TS and STS
over 500 time periods, averaged over 5000 simulations, each with an independently sampled problem instance.  For STS, we used a tolerance
parameter of $1.0$.

\begin{figure}[h!]
\subfloat[]{\includegraphics[height=2.1in]{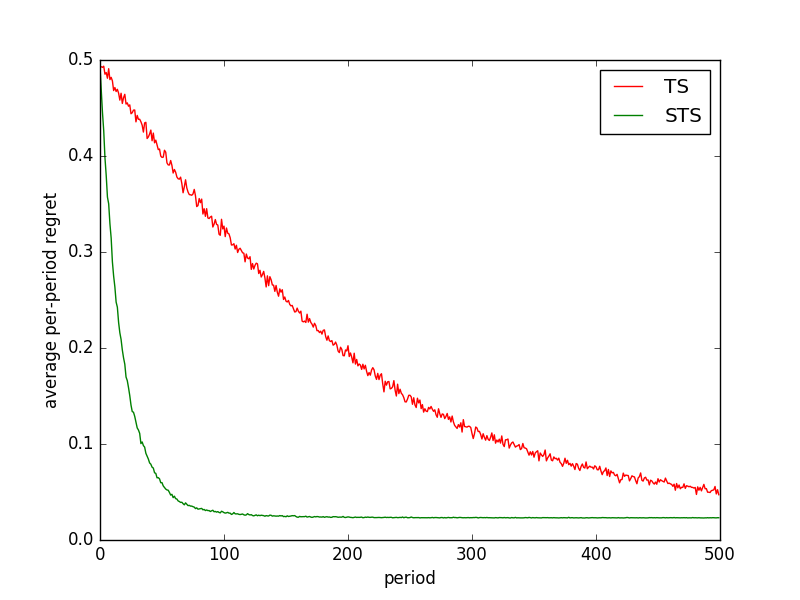}}
\subfloat[]{\includegraphics[height=2.1in]{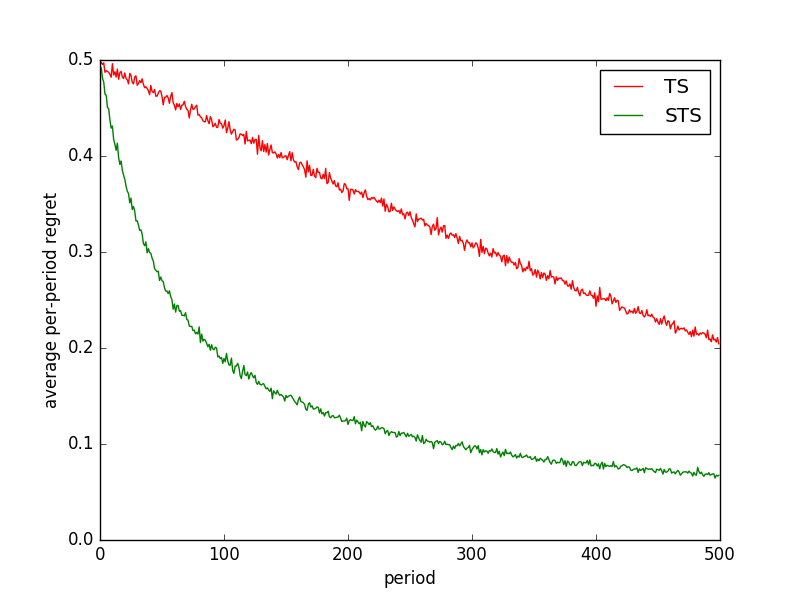}} \\
\subfloat[]{\includegraphics[height=2.1in]{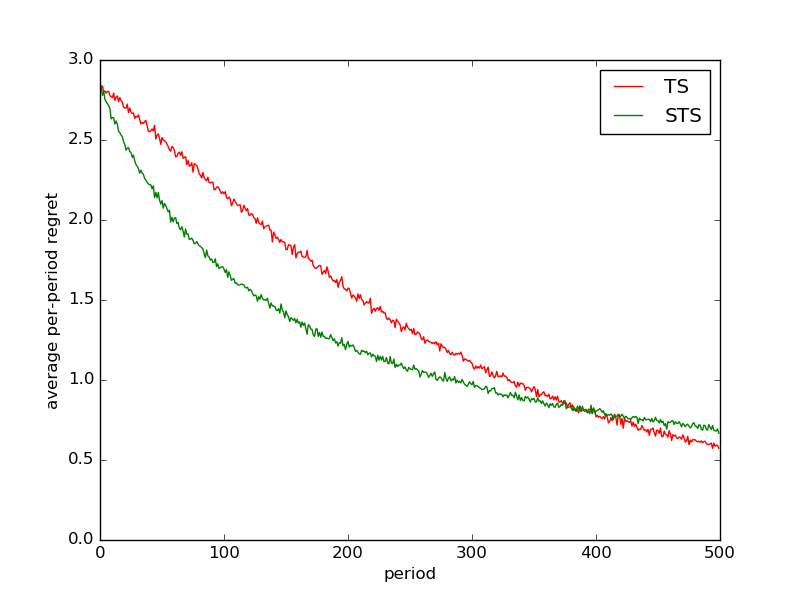}}
\subfloat[]{\includegraphics[height=2.1in]{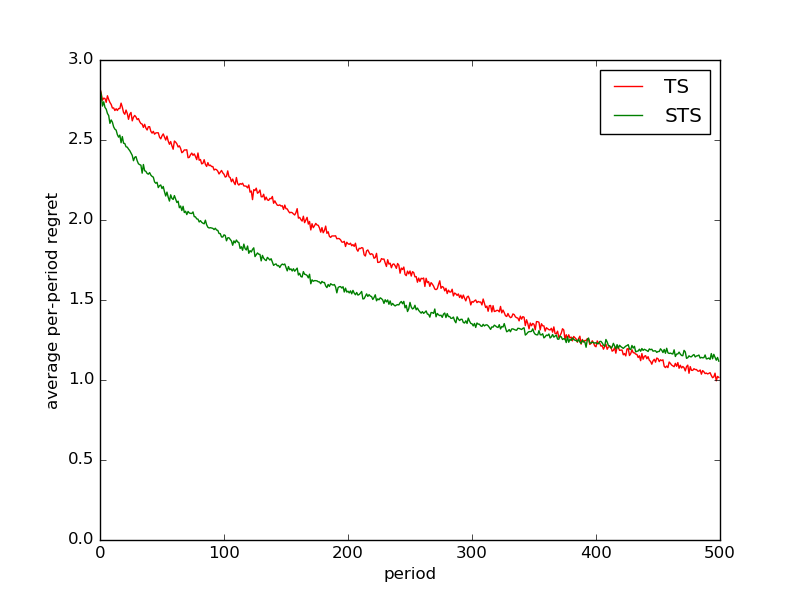}}
\caption{TS versus STS for the (a) independent uniform bandit with noiseless observations, (b) independent uniform-Bernoulli bandit,
(c) independent Gaussian bandit, and (d) linear-Gaussian bandit.}
\label{fig:computational-results}
\end{figure}

\section{A General Regret Bound}
This section provides a general discounted regret bound and a new information-theoretic analysis technique.  We'll leverage this general regret bound when analyzing STS in the next section. We begin by reviewing the information-theoretic analysis of Thompson sampling of \citet{russo2016info}, on which our analysis builds. Before beginning, let us first introduce some additional notation.
\subsection{Notation}
We denote by $\E_{t}[\cdot]=\E[\cdot | \hist ]$ the expectation operator conditioned on the history up to time $t$ and similarly define $\Prob_{t}(\cdot) = \Prob(\cdot | \hist)$. We denote the entropy of a discrete random variable $X$ by $\Ent(X)$, the mutual information between two random variables $X$ and $Y$ by $\I(X; Y)$ and the Kullback-Leibler divergence between probability distributions $P$ and $Q$ by $\D(P||Q)$.
The definitions of entropy and mutual information depend on a base measure. We use $\Ent_{t}(\cdot)$ and $\I_{t}(\cdot\, ,\cdot)$ to denote entropy and mutual-information when the base-measure is the posterior distribution $\Prob_{t}$. For example, if $X$ is a discrete random variable taking values in a set $\mathcal{X}$,
\[
\Ent_{t}(X) = - \sum_{x\in \mathcal{X}} \Prob_{t}(X=x) \log \Prob_{t}(X=x).
\]
Due to its dependence on the realized history $\hist$, $\Ent_{t}(X)$ is a random variable. The standard definition of conditional entropy integrates over this randomness, and in particular, $\E[\Ent_{t}(X)] = \Ent(X | \hist)$.

\subsection{Information Theoretic Analysis of Thompson Sampling}
The regret analysis in \cite{russo2016info} relates the regret an algorithm incurs to the information it acquires about the identity of optimal action $A^* \in \arg\max_{a} \E[R(Y_{\infty, a}) | \theta]$.  They define the \emph{information ratio} in a given period to be the ratio between the square of single-period expected regret and the information acquired about the optimal action:
\begin{equation}\label{eq: old-info-ratio}
\frac{\E_{t}[R^* - R_{t, A_t}]^2}{\I_{t}( A^* ; Y_{t,A_t} | \xi)}.
\end{equation}
It's shown that \emph{every algorithm} satisfies a bound on un-discounted expected-regret up period $T$ in terms of its average information ratio over the first $T$ periods and the entropy of the optimal action $\Ent(A^*)$. Here the information ratio roughly captures the cost-per-bit of information the algorithm acquires about the optimum, and the entropy $\Ent(A^*)$ measures the magnitude of the decision-maker's initial uncertainty about the identity of the optimal action. For a number of widely studied classes of online optimization problems, strong regret bounds for Thompson sampling can be derived by bounding the algorithm's information ratio. Subsequent work by \citet{bubeck2015bandit} and \citet{bubeck2015multi} bounds the information-ratio for bandit problems with convex reward functions.

\subsection{A Modified Information Ratio}
This section introduces a modified information ratio, which is more appropriate for time-sensitive online learning problems. As motivation, consider the infinitely-armed deterministic bandit of Section \ref{sec: infinite deterministic bandit}. While no algorithm could identify an optimal action in that example, STS is able to efficiently converge to a satisfactory level of performance. In this sense, although the algorithm can't identify the true optimum, it seems to acquire enough information to identify some high-reward action $\tilde{A}$. Building on this intuition, our information-theoretic analysis will aim to formally relate regret to the information the algorithm acquires about this $\tilde{A}$. To help ground this discussion, consider two examples of such an $\tilde{A}$ arising from different problem settings.

\begin{example}
Consider the infinitely-armed deterministic bandit of Section \ref{sec: infinite deterministic bandit}. As time progresses, STS samples a sequence of actions $(A_0, A_1, A_2,...)$. Let $\tau = \min\{ t | \theta_{A_t} \geq 1-\epsilon\}$ denote the first time the algorithm samples an action that is $\epsilon$--optimal and set $\tilde{A} := A_{\tau}$ to be the corresponding action. In this example, there are many $\epsilon$--optimal, or ''satisfactory'' actions, and $\tilde{A}$ is taken to be the first one discovered by the algorithm.
\end{example}

\begin{example}
Consider a bandit problem where mean-rewards are given by a linear function. In particular, $\A \subset \mathbb{R}^d$, and $\E[R_{t,a} | \theta] = a^T \theta$  for an unknown vector $\theta$. Suppose that $\theta \sim N(0,I)$ and $\A$ consists of $n$ vectors spread out uniformly along boundary of the $d$ dimensional unit sphere $\{a \in \mathbb{R}^d : \| a\|_2 =1\}$. The optimal action $A^* = \arg\max_{a\in \A} \theta^T a$  is then uniformly distributed over $\A$, and hence $\Ent(A^*) = \log n$. Here entropy tends to infinity the number of actions grows, and it takes an enormous amount of information to exactly identify $A^*$. For this example, we might take $\tilde{A}$ to be a coarser version of $A^*$. In particular, for $m \ll n$, let $\tilde{\A}$ consist of $m$ vectors spread out uniformly along boundary of the $d$ dimensional unit sphere $\{a \in \mathbb{R}^d : \| a\|_2 =1\}$ and let $\tilde{A} = \arg\max_{a\in \A} \theta^T a$. This can be viewed as a form of lossy-compression, where one may have $\Ent(\tilde{A}) \ll \Ent(A^*)$ but $\E[R_{t, \tilde{A}}]\geq \E[R_{t, A^*}]+\epsilon$ for some small $\epsilon>0$.
\end{example}

In each of these examples, the action $\tilde{A}$ can be viewed as some random variable taking values in the action set $\A$. In the second example, $\tilde{A}$ is a deterministic function of $\theta$, and is random only because of the randomness in $\theta$. In the first example, $\tilde{A}$ also depends on the algorithm's internal randomness, which determines the order in which actions are sampled.

To address problems of this form, we introduce the following modified information ratio. For any (random) action $\tilde{A}$ and (random) action process $\{A_t : t \in \N\}$, define
\begin{equation}\label{eq: information-ratio}
\boldGamma\left(\tilde{A}, \{A_t: t \in \N\}\right) = (1-\alpha^2) \sum_{t=0}^\infty \alpha^{2t} \E\left[\frac{\E_t[\tilde{R} - R_{t,A_t}]^2}{\I_t(\tilde{A};  Y_{t,A_t} | \xi)}\right],
\end{equation}
where $\tilde{R} = R(Y_{\infty,\tilde{A}})$.  Recall that $Y_\infty$ denotes an independent sample of the action-outcome vector.  The subscript of $\E_t$ and $\I_t$ indicates that the random variables are drawn from the probability space conditioned on $\hist$. The ratio $\E_t[\tilde{R} - R_{t,A_t}]^2/\I_t(\tilde{A};  Y_{t,A_t} | \xi)$ relates the current shortfall in performance relative to the benchmark action $\tilde{A}$ to the amount of information acquired about the benchmark action. The right-hand-side of (\ref{eq: information-ratio}) is the discounted average of these single-period ratios. The square in the discount factor $\alpha$ is consistent with the problem's original discount rate, since $\E_t[\alpha^t(\tilde{R} - R_{t,A_t})]^2 = \alpha^{2t} \E_t[\tilde{R} - R_{t,A_t}]^2$.

\subsection{General Regret Bound}

The following theorem bounds the expected discounted regret of any algorithm, or action process, $\{A_t : t\in \N\}$, in terms of the information ratio \eqref{eq: information-ratio}.
\begin{thm}
\label{th:discounted-regret}
For any action process $\{A_t : t \in \N\}$ and $\tilde{A}: \Omega \rightarrow \A$
\begin{equation}\label{eq: discounted-regret}
\E\left[\sum_{t=0}^\infty \alpha^t (R^* - R_{t,A_t})\right]
\leq \frac{\E[R^* - \tilde{R}]}{1-\alpha} + \sqrt{\frac{\Gamma\left(\tilde{A}, \{A_t: t \in \N\}\right) \Ent(\tilde{A}| \xi)}{1-\alpha^2}}.
\end{equation}
where $\tilde{R} = R(Y_{\infty,\tilde{A}})$.
\end{thm}
This bound decomposes regret into the sum of two terms; one which captures the discounted  performance shortfall of the benchmark action $\tilde{A}$ relative to $A^*$, and one which bounds the additional regret incurred while learning to identify $\tilde{A}$. Breaking things down further, the entropy $\Ent(\tilde{A}| \xi)$ measures the magnitude of the decision-maker's initial uncertainty about $\tilde{A}$, and the information ratio measures the regret incurred in reducing this uncertainty. It's worth highlighting that for any given action process, this bound holds simultaneously for all possible choices of $\tilde{A}$, and in particular, it holds for the $\tilde{A}$ minimizing the right hand side of \eqref{eq: discounted-regret}.

\subsection{Connections to Rate Distortion Theory}
In information-theory, the entropy of a source characterizes the length of an optimal lossless encoding. The celebrated rate-distortion theory  \citep[Chapter~10]{cover2012elements} characterizes the number of bits required for an encoding to be close in some loss metric. This theory resolves when it is possible to to derive a satisfactory lossy compression scheme while transmitting far less information than required for a lossless compression. At a high level, the developments in this paper represent a shift from entropy to the use of rate-distortion function. Whereas prior results depend on the entropy of $A^*$, Theorem \ref{th:discounted-regret} depends on a naturally defined rate distortion function for compressing the optimal decision $A^*$:
$$R(D) := \min_{\E[R^* - \tilde{R}] \leq D} \I(\tilde{A}; A^*).$$
When $\tilde{A}$ depends deterministically on $A^*$, $\I(\tilde{A};A^*) = \Ent(\tilde{A})$, and hence the rate-distortion function describes the optimal tradeoff between the loss in reward  $\E[R^* - \tilde{R}] $ and the entropy of $\tilde{A}$, precisely what is needed in minimizing the right hand side of \eqref{eq: discounted-regret}.

\section{Information Ratio Analysis of the Infinitely-Armed Bandit}
The general regret bound of the previous section can be instantiated on two variants of the infinite-armed bandit problem. The next subsection revisits the deterministic infinite-armed bandit of Section \ref{sec: infinite deterministic bandit}, and shows how to derive a regret bound for STS using Theorem \ref{th:discounted-regret}. Subsection \ref{subsec: infinite noisy} studies an extension of the infinite-armed bandit problem in which reward-observations are noisy. Again, in this setting Theorem \ref{th:discounted-regret} can be specialized to derive a regret bound for STS.

\subsection{Infinitely--Armed Bandit with Deterministic Observations}
We now revisit the infinitely-armed deterministic bandit problem of Section \ref{sec: infinite deterministic bandit}. By specializing  our general regret bound this setting, we will effectively recover the bound of Theorem \ref{thm: direct regret bound for STS} that was derived from direct analysis. Because there is no observation noise in this example, once STS samples an action with reward exceeding $1-\epsilon$, it will sample it in all subsequent periods. Before that point, the algorithm knows with certainty that no previously-sampled action generates reward exceeding $1-\epsilon$, and so a new action will be selected in every period.  Let $\tau = \min\{ t | \theta_{A_t} \geq 1-\epsilon\}$ denote the first time an $\epsilon$--optimal action is sampled. The next result applies the general regret bound of Theorem \ref{th:discounted-regret} to this problem with $\tilde{A} = A_{\tau}$, so the benchmark action is the first $\epsilon$--optimal action sampled by STS.

\begin{thm}\label{thm: info ratio of infinite deterministic bandit}
For any $\alpha \in (0,1)$, if STS is applied to the deterministic infinite bandit problem with tolerance $\epsilon\in (0,1)$ then
\[
\Ent(\tilde{A} | \xi)= \Ent(\tau) \quad \mathrm{and} \quad \boldGamma\left(\tilde{A}, \{A_t: t \in \N\}\right)  \leq  \frac{1}{4\epsilon \Ent(\tau)}
\]
where $\tau = \min\{ t | \theta_{A_t} \geq 1-\epsilon\}$ follows a Geometric distribution with parameter $\epsilon$ and $\tilde{A} = A_{\tau}$. This implies that if $\epsilon = \sqrt{(1-\alpha)/2}$,
\[
\E\left[\sum_{t=0}^\infty \alpha^t (R^* - R_{t,A_t})\right] \leq \sqrt{\frac{2}{1-\alpha}}.
\]
\end{thm}

\subsection{Infinitely--Armed Bandit with Noisy Observations}\label{subsec: infinite noisy}

Now consider a generalization of the problem treated in the previous section that allows for noisy observations and non-uniform priors. We again assume there is a countable action space $\mathcal{A} = \{1,2,\ldots\}$. Each action $a \in \mathcal{A}$ yields expected reward $\E[R_{t,a} | \theta] = \theta_a$ where the $\theta_a$ are drawn independently from a distribution whose support is the unit interval $[0,1]$. Assume rewards are bounded in $[0,1]$ almost surely.

We'll study the discounted regret incurred by STS with parameter $\epsilon \in (0,1)$.
Each action sampled by STS is $\epsilon$--optimal with probability $\delta \equiv \Prob(\theta_a > 1-\epsilon)$, but because observations are noisy, the algorithm may be uncertain about the quality of the actions it has sampled. The next result provides a regret bound for STS in this more complicated setting. The proof again leverages Theorem \ref{th:discounted-regret} with the benchmark action $\tilde{A}$ taken to be the first $\epsilon$--optimal action sampled by the algorithm. By bounding the problem's information-ratio, we relate the regret incurred by STS to the information it acquires about the identity of $\tilde{A}$.

\begin{thm}\label{thm:noisy regret bound} Suppose STS with tolerance parameter $\epsilon \in (0,1)$ is applied to the infinite-armed bandit with noisy observations. Then, with probability 1, there exists $t \in \N$ with $\theta_{A_t} \geq 1-\epsilon$. If $\tilde{A} = A_{\tau}$ where $\tau = \min\{ t : \theta_{A_t} \geq 1-\epsilon \}$, then
\[
\Ent( \tilde{A}| \xi) \leq 1+ \log(1/\delta) \quad \mathrm{and} \quad \boldGamma\left(\tilde{A}, \{A_t: t \in \N\}\right)  \leq 6+ 4/\delta  + (2/\delta)\log\left( \frac{1}{1-\alpha^2} \right).
\]
Together with Theorem \ref{th:discounted-regret} this implies
\begin{eqnarray*}
\E\left[\sum_{t=0}^\infty \alpha^t (R^* - R_{t,A_t})\right]
&\leq& \frac{\epsilon}{1-\alpha} + \sqrt{ \frac{\left(6+ 4/\delta  + (2/\delta)\log\left( \frac{1}{1-\alpha^2} \right) \right)(1+\log(1/\delta))   }{1-\alpha^2}} \\
&=& \tilde{O}\left( \frac{\epsilon}{1-\alpha} + \sqrt{\frac{ 1/\delta   }{1-\alpha^2}} \right).
\end{eqnarray*}
\end{thm}

\section{Conclusion}
This paper introduces satisficing Thompson sampling -- a variation of Thompson sampling that can offer vastly superior performance when the optimal action is costly to identify relative to high performing suboptimal actions. We have also developed a general information-theoretic framework for analyzing discounted regret. This framework provides a novel link between optimal decision-making with time preferences and the study of lossy data compression. Important questions remain open, but we hope this link will open up many paths for future research.

\pagebreak

\bibliography{references}

\begin{thebibliography}{17}
\providecommand{\natexlab}[1]{#1}
\providecommand{\url}[1]{\texttt{#1}}
\expandafter\ifx\csname urlstyle\endcsname\relax
  \providecommand{\doi}[1]{doi: #1}\else
  \providecommand{\doi}{doi: \begingroup \urlstyle{rm}\Url}\fi

\bibitem[Auer et~al.(2002)Auer, Cesa-Bianchi, and Fischer]{auer2002finite}
P.~Auer, N.~Cesa-Bianchi, and P.~Fischer.
\newblock Finite-time analysis of the multiarmed bandit problem.
\newblock \emph{Machine learning}, 47\penalty0 (2):\penalty0 235--256, 2002.

\bibitem[Berry et~al.(1997)Berry, Chen, Zame, Heath, and
  Shepp]{berry1997bandit}
D.~A. Berry, R.~W. Chen, A.~Zame, D.~C. Heath, and L.~A. Shepp.
\newblock Bandit problems with infinitely many arms.
\newblock \emph{The Annals of Statistics}, 25\penalty0 (5):\penalty0
  2103--2116, 1997.

\bibitem[Bonald and Proutiere(2013)]{bonald2013two}
T.~Bonald and A.~Proutiere.
\newblock Two-target algorithms for infinite-armed bandits with bernoulli
  rewards.
\newblock In C.~J.~C. Burges, L.~Bottou, M.~Welling, Z.~Ghahramani, and K.~Q.
  Weinberger, editors, \emph{Advances in Neural Information Processing Systems
  26}, pages 2184--2192. Curran Associates, Inc., 2013.

\bibitem[Bubeck and Eldan(2015)]{bubeck2015multi}
S.~Bubeck and R.~Eldan.
\newblock Multi-scale exploration of convex functions and bandit convex
  optimization.
\newblock \emph{arXiv preprint arXiv:1507.06580}, 2015.

\bibitem[Bubeck et~al.(2011)Bubeck, Munos, Stoltz, and
  Szepesv{\'a}ri]{bubeck2011xarmed}
S.~Bubeck, R.~Munos, G.~Stoltz, and {C}. Szepesv{\'a}ri.
\newblock X-armed bandits.
\newblock \emph{Journal of Machine Learning Research}, 12:\penalty0 1655--1695,
  June 2011.

\bibitem[Bubeck et~al.(2015)Bubeck, Dekel, Koren, and Peres]{bubeck2015bandit}
S.~Bubeck, O.~Dekel, T.~Koren, and Y.~Peres.
\newblock Bandit convex optimization: $\sqrt{T}$ regret in one dimension.
\newblock \emph{arXiv preprint arXiv:1502.06398}, 2015.

\bibitem[Cover and Thomas(2012)]{cover2012elements}
T.M. Cover and J.A. Thomas.
\newblock \emph{Elements of information theory}.
\newblock John Wiley \& Sons, 2012.

\bibitem[Deshpande and Montanari(2012)]{deshpande2012linear}
Y.~Deshpande and A.~Montanari.
\newblock Linear bandits in high dimension and recommendation systems.
\newblock In \emph{Communication, Control, and Computing (Allerton), 2012 50th
  Annual Allerton Conference on}, pages 1750--1754. IEEE, 2012.

\bibitem[Francetich and Kreps(2016{\natexlab{a}})]{francetich2016toolkita}
A.~Francetich and D.~M. Kreps.
\newblock Choosing a good toolkit, {I}: Formulation, heuristics, and asymptotic
  properties.
\newblock \emph{preprint}, 2016{\natexlab{a}}.

\bibitem[Francetich and Kreps(2016{\natexlab{b}})]{francetich2016toolkitb}
A.~Francetich and D.~M. Kreps.
\newblock Choosing a good toolkit, {II}: Simulations and conclusions.
\newblock \emph{preprint}, 2016{\natexlab{b}}.

\bibitem[Gittins et~al.(2011)Gittins, Glazebrook, and Weber]{gittins2011multi}
J.~Gittins, K.~Glazebrook, and R.~Weber.
\newblock \emph{Multi-Armed Bandit Allocation Indices}.
\newblock John Wiley \& Sons, Ltd, 2011.
\newblock ISBN 9780470980033.

\bibitem[Kleinberg et~al.(2008)Kleinberg, Slivkins, and
  Upfal]{kleinberg2008multi}
R.~Kleinberg, A.~Slivkins, and E.~Upfal.
\newblock Multi-armed bandits in metric spaces.
\newblock In \emph{Proceedings of the 40th ACM Symposium on Theory of
  Computing}, 2008.

\bibitem[Rusmevichientong and Tsitsiklis(2010)]{rusmevichientong2010linearly}
P.~Rusmevichientong and J.N. Tsitsiklis.
\newblock Linearly parameterized bandits.
\newblock \emph{Mathematics of Operations Research}, 35\penalty0 (2):\penalty0
  395--411, 2010.

\bibitem[Russo and Van~Roy(2016)]{russo2016info}
D.~Russo and B.~Van~Roy.
\newblock An information-theoretic analysis of {T}hompson sampling.
\newblock \emph{Journal of Machine Learning Research}, 17\penalty0
  (68):\penalty0 1--30, 2016.

\bibitem[Ryzhov et~al.(2012)Ryzhov, Powell, and Frazier]{ryzhov2012knowledge}
I.O. Ryzhov, W.B. Powell, and P.I. Frazier.
\newblock The knowledge gradient algorithm for a general class of online
  learning problems.
\newblock \emph{Operations Research}, 60\penalty0 (1):\penalty0 180--195, 2012.

\bibitem[Thompson(1933)]{thompson1933}
W.R. Thompson.
\newblock On the likelihood that one unknown probability exceeds another in
  view of the evidence of two samples.
\newblock \emph{Biometrika}, 25\penalty0 (3/4):\penalty0 285--294, 1933.

\bibitem[Wang et~al.(2009)Wang, Audibert, and Munos]{wang2009algorithms}
Y.~Wang, J.-Y. Audibert, and R.~Munos.
\newblock Algorithms for infinitely many-armed bandits.
\newblock In D.~Koller, D.~Schuurmans, Y.~Bengio, and L.~Bottou, editors,
  \emph{Advances in Neural Information Processing Systems 21}, pages
  1729--1736. Curran Associates, Inc., 2009.

\end{thebibliography}
\bibliographystyle{plainnat}

\vspace*{\fill}

\pagebreak
\appendix

\section{Proof of Theorem \ref{thm: regret of TS}: Regret of TS on the Infinitely-Armed Deterministic Bandit}

\begin{proof}
In every period $t$, TS samples a previously un-sampled action $A_t \notin \{A_1,...,A_{t-1}\}$, which generates expected reward $\E[\theta_{A_t}] = \E[\theta_1] = 1/2$.  The optimal expected reward is 1, and therefore the expected discounted-regret of TS is
\[
\sum_{t=0}^{\infty} \alpha^{t} (1-1/2) = \frac{1}{2(1-\alpha)}.
\]
\end{proof}

\section{Proof of Theorem \ref{thm: direct regret bound for STS}: Direct Analysis of the Infinitely-Armed Deterministic Bandit}

\begin{proof}Let $\tau = \min\{t:\theta_{A_t} \geq 1-\epsilon\}$.
\begin{eqnarray*}
\E\left[\sum_{t=0}^\infty \alpha^t (R^* - R_t)\right]
&=& \E\left[\sum_{t=0}^\infty \alpha^t (1 - R_t)\right] \\
&=& \E\left[\E\left[\sum_{t=0}^{\tau-1} \alpha^t (1-R_t) + \sum_{t=\tau}^\infty \alpha^t (1-R_{t,A_t}) \Big| \tau \right]\right] \\
&=& \E\left[\frac{(1-\alpha^\tau) (1-\epsilon)}{2 (1-\alpha)} + \frac{\alpha^\tau \epsilon}{2(1-\alpha)}\right] \\
&=& \E\left[\frac{(1-\alpha^\tau) (1-\epsilon)}{2 (1-\alpha)} + \frac{\epsilon}{2(1-\alpha)} - \frac{(1-\alpha^\tau) \epsilon}{2(1-\alpha)}\right] \\
&=& \E\left[\frac{\epsilon}{2(1-\alpha)} + \frac{(1-\alpha^\tau) (1-2\epsilon)}{2(1-\alpha)}\right].
\end{eqnarray*}
Note that
$$\E[1 - \alpha^\tau] = 1 - \sum_{t=0}^\infty \epsilon (1-\epsilon)^t \alpha^t = 1 - \frac{\epsilon}{1 - \alpha + \epsilon \alpha} = \frac{1-\alpha-\epsilon+\epsilon\alpha}{1-\alpha+\epsilon\alpha}= \frac{(1-\alpha)(1-\epsilon)}{1-\alpha(1-\epsilon)}.$$
Therefore
\[
\E\left[ \frac{(1-\alpha^\tau) (1-2\epsilon)}{2(1-\alpha)}\right]= \frac{(1-\epsilon)(1-2\epsilon)}{2(1-\alpha(1-\epsilon))}
= \frac{(1-\epsilon)(1-2\epsilon)}{2(\epsilon+(1-\alpha)(1-\epsilon))}.
\]
Now consider an upper bound that follows from choosing $\epsilon$ as a function of $\alpha$. We can simplify our upper bound on regret as
\[
\frac{\epsilon}{2(1-\alpha)} + \frac{(1-\epsilon)(1-2\epsilon)}{2(\epsilon+(1-\alpha)(1-\epsilon))} \leq \frac{\epsilon}{2(1-\alpha)} + \frac{1}{2\epsilon}
\]
The minimizer of the right hand side is $\epsilon^* = \sqrt{1-\alpha}$. Plugging this in shows that
 Thompson sampling with a confidence bonus of $\epsilon^*$ satisfies the discounted regret bound
$$\E\left[\sum_{t=0}^\infty \alpha^t (R^* - R_{t,A_t})\right] \leq \frac{1}{\sqrt{1-\alpha}}.$$
\end{proof}

\section{Proof of Theorem \ref{th:discounted-regret}}
\begin{proof}
We first show that entropy bounds the expected accumulation of mutual-information. By the chain rule for mutual information, for any $T$,
\begin{eqnarray*}
\E\left[\sum_{t=0}^{T-1} \I_t(\tilde{A}; Y_{t,A_t} | \xi)\right]
&=& \sum_{t=0}^{T-1} \I(\tilde{A}; Y_{t,A_t} | \xi,  H_t) \\
&=& \sum_{t=0}^{T-1} \I(\tilde{A}; Y_{t,A_t} | \xi, A_0, Y_{0,A_0}, \ldots, A_{t-1}, Y_{t-1,A_{t-1}}) \\
&=& \sum_{t=0}^{T-1} \I(\tilde{A}; Y_{t,A_t} | \xi, A_0, Y_{0,A_0}, \ldots, A_{t-1}, Y_{t-1,A_{t-1}}, A_{t}) \\
&=& \I(\tilde{A}; (A_0, Y_{0,A_0}, \ldots, A_t, Y_{t,A_t})| \xi) \\
&=& \Ent(\tilde{A} | \xi) - \Ent(\tilde{A} | A_0, Y_{0,A_0}, \ldots, A_t, Y_{t,A_t}, \xi) \\
&\leq& \Ent(\tilde{A}| \xi).
\end{eqnarray*}
Taking a the limit as $T \rightarrow \infty$ implies
\[
\E\left[\sum_{t=0}^{\infty} \I_t(\tilde{A}; Y_{t,A_t} | \xi)\right] = \lim_{T\rightarrow \infty} \E\left[\sum_{t=0}^{T} \I_t(\tilde{A}; Y_{t,A_t} | \xi)\right]  \leq \Ent(\tilde{A} | \xi),
\]
where the monotone convergence theorem justifies the exchange of limit and expectation. Now, fix any $\tilde{A}$ and $\{A_t : t \in \N\}$, and let
\[
\Gamma_{t} \equiv \frac{\E_t[\tilde{R} - R_{t,A_t}]^2}{\I_t(\tilde{A}; Y_{t,A_t}| \xi ))}
\]
denote the (random) information ratio at time $t$ under the benchmark action $\tilde{A}$ and action process $\{A_t : t \in \N\}$. Then we have
\begin{eqnarray*}
\E\left[\sum_{t=0}^\infty \alpha^t (R^* - R_{t,A_t})\right]
&=& \E\left[\sum_{t=0}^\infty \alpha^t (R^* - \tilde{R})\right] + \E\left[\sum_{t=0}^\infty \alpha^t (\tilde{R} - R_{t, A_t})\right] \\
&=& \frac{\E\left[R^* - \tilde{R}\right]}{1-\alpha} + \E\left[\sum_{t=0}^\infty \sqrt{\alpha^{2t} \Gamma_{t}} \sqrt{\I_t(\tilde{A}; Y_{t,A_t}) | \xi )}\right] \\
&\leq& \frac{\E\left[R^* - \tilde{R}\right]}{1-\alpha} + \sqrt{\E\left[\sum_{t=0}^\infty \alpha^{2t} \Gamma_{t}\right]} \sqrt{\E\left[\sum_{t=0}^\infty \I_t(\tilde{A}; Y_{t,A_t}| \xi )\right]} \\
&\leq& \frac{\E\left[R^* - \tilde{R}\right]}{1-\alpha} + \sqrt{\E\left[\sum_{t=0}^\infty \alpha^{2t} \Gamma_{t}\right]}\sqrt{\Ent(\tilde{A} | \xi )} \\
&=& \frac{\E\left[R^* - \tilde{R}\right]}{1-\alpha} + \sqrt{\frac{\boldGamma\left(\tilde{A}, \{A_t: t \in \N\}\right) \Ent(\tilde{A}| \xi)}{1-\alpha^2}},
\end{eqnarray*}
where the first inequality follows from the Cauchy-Schwarz inequality and the second was established earlier in this proof.
\end{proof}

\section{Proof of Theorem \ref{thm: info ratio of infinite deterministic bandit}: Information-Ratio Analysis of Infinitely--Armed Deterministic Bandit}

\begin{lemma}
Under STS with tolerance $\epsilon\in (0,1)$ in  the infinitely--armed deterministic bandit problem, $\tau = \min\{ t | \theta_{A_t} \geq 1-\epsilon\}$ follows a Geometric distribution with parameter $\epsilon$, and if $\tilde{A} = A_{\tau}$ then
$$\I_t(\tilde{A}; Y_{t,A_t} | \xi ) = \left\{\begin{array}{ll}
\epsilon \Ent(\tau) \qquad & \text{\rm if } \E_t[R^*-R_{t,A_t}] > \epsilon \\
0 \qquad & \text{\rm otherwise.}
\end{array}
\right.$$
\end{lemma}
\begin{proof}
As time progresses, STS samples actions $A_1, A_2, A_3...$. At each time $t <= \tau$, it selects a previously un-sampled action $A_t \notin \{A_1,...A_{t-1}\}$. It selects the actions $A_t = A_\tau$ in each period $t>\tau$. Because $\Prob(\theta_a \geq 1-\epsilon) = \epsilon$ for each $a$,  we have that $\tau$ follows a Geometric distribution with parameter $\epsilon$. Conditioned on $\tau \geq t$, the identity of $A_t$ is determined by the algorithm's internal random bits $\xi$. That is, the order of the new actions sampled by the algorithm is a function only of $\xi$. Therefore, $H(\tilde{A}| \xi) = H(\tau)$.

Under STS, if $\E_t[R^*-R_{t,A_t}] \leq \epsilon$ then $A_t = \tilde{A}$, and $\Gamma_{t} =0$ since $\E_t[\tilde{R}-R_{t,A_t}] = 0.$
On the other hand, if $\E_t[R^*-R_{t,A_t}] > \epsilon$ then $A_t \notin \{A_1,...A_{t-1}\}$ and
\begin{eqnarray*}
\I_t(\tilde{A}; Y_{t,A_t} | \xi)
&=& \Ent_t(\tilde{A} | \xi) - \Ent_t(\tilde{A} | \xi,Y_{t,A_t}) \\
&=& \Ent_t(\tilde{A} | \xi ) - \int_{y=0}^1 \Ent_t(\tilde{A} | \xi, Y_{t,A_t} \in dy) \\
&=& \Ent_t(\tilde{A}| \xi) - \int_{y=0}^{1-\epsilon} \Ent_t(\tilde{A} | \xi, Y_{t,A_t} \in dy) - \int_{y=1-\epsilon}^1 \Ent_t(\tilde{A} | \xi, Y_{t,A_t} \in dy)\\
&=& \Ent_t(\tilde{A}| \xi) - \int_{y=0}^{1-\epsilon} \Ent_t(\tilde{A} | \xi, Y_{t,A_t} \in dy) \\
&=& \Ent_t(\tilde{A} | \xi) - (1-\epsilon) \Ent(\tilde{A} | \xi) \\
&=& \epsilon \Ent_t(\tilde{A} | \xi) \\
&=& \epsilon \Ent_t(\tau).
\end{eqnarray*}
\end{proof}

Together with the previous lemma, our general regret bound implies Theorem \ref{thm: info ratio of infinite deterministic bandit}.

\begin{proof}[Proof of Theorem \ref{thm: info ratio of infinite deterministic bandit}] We have
\begin{eqnarray*}
\boldGamma\left(\tilde{A}, \{A_t: t \in \N\}\right)
&=& (1-\alpha^2) \E\left[\sum_{t=0}^\infty \alpha^{2t} \frac{\E_t[\tilde{R} - R_{t,A_t}]^2}{\I_t(\tilde{A};  Y_{t,A_t}| \xi )}\right] \\
&=& (1-\alpha^2) \E\left[\sum_{t=0}^{\tau-1} \alpha^{2t} \frac{(1-\epsilon)^2}{4 \epsilon \Ent(\tau)}
+ \sum_{t=\tau}^\infty \alpha^{2t} \frac{0^2}{0}\right] \\
&=& \E\left[1-\alpha^{2\tau}\right] \frac{(1-\epsilon)^2}{4 \epsilon \Ent(\tau)} \\
&=&  \frac{(1-\epsilon)^2 (1 - \alpha^2 + \epsilon\alpha^2 - \epsilon)}{4 \epsilon \Ent(\tau) (1 - \alpha^2 + \epsilon \alpha^2)} \\
&\leq & \frac{1}{4\epsilon \Ent(\tau)}.
\end{eqnarray*}
It follows from Theorem \ref{th:discounted-regret} that
\begin{eqnarray*}
\E\left[\sum_{t=0}^\infty \alpha^t (R^* - R_{t,A_t})\right]
&\leq& \frac{\E[R^* - \tilde{R}]}{1-\alpha} + \sqrt{\frac{\boldGamma\left(\tilde{A}, \{A_t: t \in \N\}\right) I(\theta; \tilde{A} | \xi)}{1-\alpha^2}} \\
&\leq& \frac{\epsilon}{1-\alpha} + \sqrt{\frac{(1-\epsilon)^2 (1 - \alpha^2 + \epsilon\alpha^2 - \epsilon)}{4 \epsilon (1 - \alpha^2 + \epsilon \alpha^2) (1-\alpha^2)}} \\
&=& \frac{\epsilon}{1-\alpha} + \sqrt{\frac{1- \epsilon^3}{4 \epsilon^2 + \epsilon(1-\epsilon)(1-\alpha^2))}}\\
& \leq & \frac{\epsilon}{1-\alpha} + \frac{1}{ 2\epsilon}.
\end{eqnarray*}
Now we consider an upper bound that follows from choosing $\epsilon$ as a function of $\alpha$. The minimizer of $\frac{\epsilon}{1-\alpha} + \frac{1}{ 2\epsilon}$ is $\epsilon^* = \sqrt{(1-\alpha)/2}$. If $\{A_t : t \in \N\}$ is generated STS with parameter $\epsilon^*$, the bound on regret becomes
\[
\E\left[\sum_{t=0}^\infty \alpha^t (R^* - R_{t,A_t})\right] \leq \sqrt{\frac{2}{1-\alpha}}.
\]

\end{proof}

\section{Proof of Theorem \ref{thm:noisy regret bound}: Information Ratio Analysis of the Infinitely-Armed Bandit with Noisy Observations}

The proof of Theorem \ref{thm:noisy regret bound} leverages the probability matching property of STS highlighted in Section \ref{sec: algorithms}. Recall that $\tilde{A} = A_{\tau}$ where $\tau = \min\{t | \theta_{A_t} \geq \epsilon\}$. Throughout this proof, let $\A_t \equiv \{A_1, A_2,...A_{t-1}\}$ denote the set of previously sampled actions. Under STS, $\Prob(A_t = a | \hist) = \Prob(\tilde{A} = a | \hist)$ for all $a \in \A_{t}$, and $\Prob(A_t \notin \A_{t} | \hist) = \Prob( \tilde{A} \notin \A_t | \hist)$. The algorithm essentially performs a kind of probability matching on $\tilde{A}$.

\begin{thmn}[\ref{thm:noisy regret bound}] Suppose STS with tolerance parameter $\epsilon \in (0,1)$ is applied to the infinite-armed bandit with noisy observations. Then, with probability 1, there exists $t \in \N$ with $\theta_{A_t} \geq 1-\epsilon$. If $\tilde{A} = A_{\tau}$ where $\tau = \min\{ t : \theta_{A_t} \geq 1-\epsilon \}$, then
\[
\Ent( \tilde{A}| \xi) \leq 1+ \log(1/\delta) \quad \mathrm{and} \quad \boldGamma\left(\tilde{A}, \{A_t: t \in \N\}\right)  \leq 6+ 4/\delta  + (2/\delta)\log\left( \frac{1}{1-\alpha^2} \right).
\]
Together with Theorem \ref{th:discounted-regret} this implies
\begin{eqnarray*}
\E\left[\sum_{t=0}^\infty \alpha^t (R^* - R_{t,A_t})\right]
&\leq& \frac{\epsilon}{1-\alpha} + \sqrt{ \frac{\left(6+ 4/\delta  + (2/\delta)\log\left( \frac{1}{1-\alpha^2} \right) \right)(1+\log(1/\delta))   }{1-\alpha^2}} \\
&=& \tilde{O}\left( \frac{\epsilon}{1-\alpha} + \sqrt{\frac{ 1/\delta   }{1-\alpha^2}} \right).
\end{eqnarray*}
\end{thmn}

 We begin with a lemma establishing that with probability 1 STS will eventually sample an $\epsilon$--optimal action.  At an intuitive level, this result follows from the algorithm's probability matching property, which guarantees that whenever its likely that no $\epsilon$--optimal action has been sampled previously, the algorithm is likely to select a previously un-sampled action. With probability $\delta$ this new action is $\epsilon$--optimal.
\begin{lemma}
If STS with tolerance parameter $\epsilon \in (0,1)$ is applied to the infinite-armed bandit with noisy observations, then, with probability 1, there exists $t \in \N$ with $\theta_{A_t} \geq 1-\epsilon$.
\end{lemma}
\begin{proof}
Our goal is to show $\Prob(\tau<\infty)=1$ where
\[
\tau = \inf\{ t: \theta_{A_t} \geq 1 - \epsilon \}.
\]
By the so-called continuity of measure,
\[
\Prob(\tau< \infty) = \lim_{t\rightarrow \infty} \Prob(\tau \leq t) = 1 - \lim_{t\rightarrow \infty} \Prob(\tau \geq t).
\]
Now set
\[
\beta \equiv \lim_{t\rightarrow \infty} \Prob(\tau \geq t)
\]
Because $\Prob(\tau \geq t)$ is a decreasing bounded sequence, this limit exists, and $\beta = \inf_{t \in \N}\Prob(\tau \geq t)$.  The proof shows $\beta= 0$.

By the probability matching property of STS $\Prob(A_t \notin \A_t | \hist)= \Prob(\tau \geq t | \hist) $. Then, by the definition of $\tau$ and the independence among the components of $(\theta_1, \theta_2, ...)$
\begin{eqnarray*}
\Prob(\tau = t | \hist) &=& \Prob(\tilde{A} \notin \A_t \wedge A_t \notin \A_t | \hist)\delta  \\
&=& \Prob(\tilde{A} \notin \A_t | \hist)^2 \delta \\
&=& \Prob(\tau \geq t | \hist)^2 \delta.
\end{eqnarray*}
Taking expectations implies
\[
\Prob(\tau = t) =\E\left[\Prob(\tau \geq t | \hist)^2\right] \delta \geq  \E\left[\Prob(\tau \geq t | \hist)\right]^2 \delta = \Prob(\tau \geq t )^2 \delta.
\]
Then
\[
\Prob(\tau \geq t)-\Prob(\tau \geq t+1) = \Prob(\tau=t) \geq  \Prob(\tau \geq t)^2\delta.
\]
Since $\Prob(\tau \geq t)$ converges,
\[
0 = \lim_{t\rightarrow \infty} (\Prob(\tau \geq t)-\Prob(\tau \geq t+1)) \geq \lim_{t\rightarrow \infty} \Prob(\tau \geq t)^2 \delta = \beta^2 \delta.
\]
Since $\beta \in [0,1]$ by definition, this implies $\beta=0$.
\end{proof}

The remaining proof will follow from a sequence of lemmas. We now bound the entropy of $\tilde{A}$.

\begin{lemma}
\[
\Ent(\tilde{A}| \xi) \leq 1+ \log(1/\delta)
\]
\end{lemma}
\begin{proof}
Because the order in which new actions are sampled is completely determined given $\xi$, $\Ent(\tilde{A} | \xi) = \Ent(N)$ where $N\sim {\rm Geom}(\delta)$ is a geometric random variable. This implies
\begin{eqnarray*}
\Ent(\tilde{A} | \xi)
&=& \Ent\left( N \right) \\
&=& -\sum_{k=1}^{\infty} \delta (1-\delta)^{k-1} \log(\delta (1-\delta)^{k-1}) \\
&=& -\sum_{k=1}^{\infty}\delta (1-\delta)^{k-1}\log(\delta) - \sum_{k=1}^{\infty}\delta(1-\delta)^{k-1}\log((1-\delta)^{k-1})\\
&=& \sum_{k=1}^{\infty}\Prob(N=k)\log(1/\delta)  - \log(1-\delta) \sum_{k=1}^{\infty} \delta (1-\delta)^{k-1}(k-1) \\
&=& \log(1/\delta) + \log\left( \frac{1}{1-\delta} \right)(\E[N]-1)\\
&=& \log(1 /\delta) + \log\left(1 + \frac{\delta}{1-\delta}  \right)\left( \frac{1-\delta}{\delta} \right) \\
&\leq & 1+ \log(1/\delta).
\end{eqnarray*}
\end{proof}

 The bound on entropy yields a regret bound when combined with a bound the information ratio. The next lemma gives bounds the one-step information ratio.
\begin{lemma}
\[
\frac{\E_{t}[\theta_{\tilde{A}} - \theta_{A_t}]^2}{\I_{t}(\tilde{A} ; Y_{t,A_t} | \xi)} \leq 2 |\A_{t}| +2/\delta
\]
where $\A_{t} = \cup_{s=1}^{t-1}\{ A_s \}$ is the set of actions that were sampled before period $t$, and  $\delta \equiv \Prob(\theta_{i} \geq 1-\epsilon)$ is the prior probability an arm is $\epsilon$--optimal.
\end{lemma}
\begin{proof}
Define
\[
L \equiv \E[\theta_{i} | \theta_{i} \geq 1-\epsilon] -\E[\theta_{i}]
\]
and
\[
\delta \equiv \Prob(\theta_{i} \geq 1-\epsilon).
\]
Here $\delta$ is the probability an unsampled arm is $\epsilon$ optimal, and $L$ is the difference between the expected reward of an $\epsilon$ optimal arm and that of an arm sampled uniformly at random. In the case where $\theta_i \sim {\rm Unif}(0,1)$, $\delta=\epsilon$ and $L= (1-\epsilon)/2$.

We can write expected regret as
\begin{eqnarray*}
\E_{t}[\theta_{\tilde{A}} - \theta_{A_t}] &=& \sum_{a \in \A}\Prob_{t}(\tilde{A}=a)\E_{t}[\theta_{a}|\tilde{A}=a] -  \sum_{a \in \A}\Prob_{t}(A_t=a)\E_{t}[\theta_{a}] \\
&=&\sum_{a \in \A_{t}}\Prob_{t}(\tilde{A}=a)\left(\E_{t}[\theta_{a}|\tilde{A}=a]-\E_{t}[\theta_{a}]\right)\\
&&+ \sum_{a \notin \A_{t}}\Prob_{t}(\tilde{A}=a)\E_{t}[\theta_a | \tilde{A}=a]-  \sum_{a \notin \A_t}\Prob_{t}(A_t=a)\E_{t}[\theta_{a}]\\
&=& \sum_{a \in \A_{t}}\Prob_{t}(\tilde{A}=a)\left(\E_{t}[\theta_{a}|\tilde{A}=a]-\E_{t}[\theta_{a}]\right)
+ \Prob_{t}(\tilde{A} \notin \A_{t})(\E[\theta_a | \theta_a \geq 1-\epsilon]- \E[\theta_a]) \\
&=& \underbrace{\sum_{a \in \A_{t}}\Prob_{t}(\tilde{A}=a) \left(\E_{t}[\theta_{a}|\tilde{A}=a] - \E_{t}[\theta_{a}]\right)}_{\Delta_{t,1}} + \underbrace{\Prob_{t}(\tilde{A} \notin \A_{t})L}_{\Delta_{t,2}}
\end{eqnarray*}
This decomposes regret into the sum of two terms: one which captures the regret due to suboptimal selection within the set of previously sampled actions $\A_{t}$, and one due to the remaining possibility that none of the sampled actions are $\epsilon$ optimal. The proof develops a similar decomposition for mutual information, and then lower bounds both terms.

We can express mutual information as follows:
\begin{eqnarray*}
\I_{t}(\tilde{A} ; Y_{t,A_t} | \xi) &=&\sum_{a\in \A} \Prob_{t}(A_t=a) \I_{t}(\tilde{A} ; Y_{t,a} | A_t =a) \\
 &=&\sum_{a\in \A} \Prob_{t}(\tilde{A}=a) \I_{t}(\tilde{A} ; Y_{t,a} | A_t =a) \\
&=& \sum_{a\in \A_{t}} \Prob_{t}(\tilde{A}=a)\I_{t}(\tilde{A} ; Y_{t,a}) + \Prob_{t}(\tilde{A} \notin \A_{t})\I_{t}(\tilde{A}; Y_{t, a_{N}} | A_t = a_{N})
\end{eqnarray*}
where $a_N \in \A_{t}^c$ is an arbitrary action that has not yet been sampled. (N stands for ``new'') Now, using the shorthand $P_{t}(X) = \Prob_{t}(X \in \cdot)$  to denote the posterior distribution of a random variable $X$, we have
\begin{eqnarray*}
\I_{t}(\tilde{A}; Y_{t, a_{N}} | A_t = a_{N}) &=& \sum_{a \in \A}\Prob_{t}(\tilde{A} = a | A_t = a_{N})\D\left(P_{t}(Y_{t, a_N}  | \tilde{A} = a, A_t =a_{N}) || P_{t}(Y_{t, a_N})   \right) \\
&\geq&  \Prob_{t}(\tilde{A} = a_{N} | A_t = a_{N})\D\left(P_{t}(Y_{t, a_N}  | \tilde{A} = a_N, A_t =a_{N}) || P_{t}(Y_{t, a_N})   \right) \\
&=& \Prob_{t}(\tilde{A} = a_{N} | A_t = a_{N})\D\left(P_{t}(Y_{t, a_N}  | \theta_{a_N}\geq 1-\epsilon) || P_{t}(Y_{t, a_N})   \right) \\
&\geq& 2\Prob_{t}(\tilde{A} = a_{N} | A_t = a_{N})\left( \E_{t}[R(Y_{t, a_N}) | \theta_{a_N} \geq 1-\epsilon] - \E_{t}[R(Y_{t, a_N})]  \right)^2  \\
&=& 2\Prob_{t}(\tilde{A} = a_{N} | A_t = a_{N})L^2 \\
&=& 2\Prob_{t}(\tilde{A} \notin \A_{t})\Prob(\tilde{A} = a_{N} | A_t = a_{N}, \tilde{A} \notin \A_{t}) L^2\\
&=& 2\Prob_{t}(\tilde{A} \notin \A_{t}) \delta L^2.
\end{eqnarray*}
Following the analysis from \cite{russo2016info} shows
\begin{eqnarray*}
\sum_{a\in \A_{t}} \Prob_{t}(\tilde{A}=a)\I_{t}(\tilde{A} ; Y_{t,a}) &=& \sum_{a\in \A_{t}} \Prob_{t}(\tilde{A}=a) \sum_{\tilde{a} \in \A} \D\left(P_{t}(Y_{t,a}|| \tilde{A}=\tilde{a}) || P_{t}(Y_{t,a}) \right)  \\
&\geq & \sum_{a\in \A_{t}} \Prob_{t}(\tilde{A}=a)^2 \D\left(P_{t}(Y_{t,a}|| \tilde{A}=a) || P_{t}(Y_{t,a}) \right) \\
&\geq & 2\sum_{a\in \A_{t}} \Prob_{t}(\tilde{A}=a)^2 \left(\E_{t}[\theta_{a} | \tilde{A}=a ] - \E_{t}[\theta_{a}]   \right)^2  \\
& \geq  & \frac{2}{|\A_{t}|} \left( \sum_{a \in \A_{t}} \Prob_{t}( \tilde{A} = a)\left(\E_{t}[\theta_{a} | \tilde{A}=a ] - \E_{t}[\theta_{a}]   \right) \right)^2.
\end{eqnarray*}
Therefore
\[
\I_{t}(\tilde{A} ; Y_{t,A_t} | \xi) \geq \underbrace{\frac{2}{|\A_{t}| } \left( \sum_{a \in \A_{t}} \Prob_{t}( \tilde{A} = a)\left(\E_{t}[\theta_{a} | \tilde{A}=a ] - \E_{t}[\theta_{a}]   \right) \right)^2}_{G_{t,1}} +\underbrace{ 2\Prob_{t}(\tilde{A} \notin \A_{t})^2 \delta L^2}_{G_{t,2}},
\]
is lower bounded by the sum of two terms: one which captures the information gain due to refining knowledge about previously sampled actions, and one that captures the expected information gathered about previously unexplored actions.

To bound the information ratio we'll separately consider two cases. If $\Delta_{1} \geq \Delta_{2}$,
then
\[
\frac{\E_{t}[\theta_{\tilde{A}} - \theta_{A_t}]^2}{\I_{t}(\tilde{A} ; Y_{t,A_t} | \xi)} \leq \frac{(2\Delta_{t,1})^2}{G_{t,1}+G_{t,2}} \leq \frac{4(\Delta_{t,1})^2}{G_{t,1}} = 2 |\A_{t}|.
\]
If instead $\Delta_{1} < \Delta_{2}$, then
\[
\frac{\E_{t}[\theta_{\tilde{A}} - \theta_{A_t}]^2}{\I_{t}(\tilde{A} ; Y_{t,A_t} | \xi)} \leq \frac{(2\Delta_{t,2})^2}{G_{t,1}+G_{t,2}} \leq \frac{4(\Delta_{t,2})^2}{G_{t,2}} = \frac{2}{\delta}.
\]
This shows
\[
\frac{\E_{t}[\theta_{\tilde{A}} - \theta_{A_t}]^2}{\I_{t}(\tilde{A} ;
\theta ; Y_{t,A_t} | \xi)} \leq 2 |\A_{t}| +2/\delta.
\]
\end{proof}

We'd now like to use the previous result to bound
\begin{eqnarray*}
\boldGamma\left(\tilde{A}, \{A_t: t \in \N\}\right) &=& (1-\alpha^2) \sum_{t=0}^\infty \alpha^{2t} \E\left[\frac{\E_t[\tilde{R} - R_{t,A_t}]^2}{\I_t( \tilde{A};  Y_{t,A_t} | \xi )}\right] \\
&\leq & 2/\delta + 2(1-\alpha^2) \sum_{t=0}^\infty \alpha^{2t} \E[ |\A _{t}|].
\end{eqnarray*}
We begin by bounding $\E[|\A_{t}|].$
\begin{lemma}
$|\A_{0}|=0$ and for each  $T \in \N,$  $\E[|\A_{T}|] \leq 2+ \log(T)/\delta$.
\end{lemma}
\begin{proof}
Let $\tau_{k} = \min\{ t \leq T | |\A_{t}| \geq k   \}$ denote the first period before $T$ in which $k$ actions have been sampled. Then
\begin{eqnarray*}
\E[|\A_{T}|]  &=& \E[|\A_{\tau_{k}}|]+ \E[|\A_{T}|- |\A_{\tau_{k}}|]\\
&\leq & \E[|\A_{\tau_{k}}|]+ \E[|\A_{\tau_{k}+T}|- |\A_{\tau_{k}}|]\\
&\leq & k + \E \sum_{t=\tau_{k}}^{\tau_{k}+T-1}  \mathbf{1}(A_t \notin \A_{t}  ) \\
&=& k + \E \sum_{s=0}^{T-1} \Prob(A_{\tau_{k}+s} \notin \A_{\tau_{k}+s} | H_{\tau_k +s} ) \\
&=&  k + \E \sum_{s=0}^{T-1} \Prob(\tilde{A} \notin \A_{\tau_{k}+s} | H_{\tau_k +s} ) \\
&= & k + \sum_{s=0}^{T-1} \Prob(\tilde{A} \notin \A_{\tau_{k}+s}) \\
& \leq & k + T \Prob(\tilde{A} \notin \A_{\tau_{k}}) \\
&= & k + T \Prob( {\rm Geom}(\delta) > k)  \\
&=& k + T( 1- \delta)^{k} \\
&\leq &  k +Te^{-\delta k}.
\end{eqnarray*}
Choosing $k = \lceil \log(T) / \delta \rceil \leq 1 +\log(T) / \delta,$ implies
\[
\E[ |\A_{T}|] \leq  2 + \log(T) / \delta.
\]
\end{proof}
The next technical lemma shows $\sum_{t=1}^{\infty} \gamma^{-t} \log(t)=O((1/\gamma)\log(1/\gamma)).$

\begin{lemma} For any $\gamma \in (0,1)$,
\[
\sum_{t=1}^{\infty} \gamma^{-t} \log(t)\leq \frac{1}{1-\gamma} \left[ 1+ \log\left(\frac{1}{1-\gamma}\right)\right].
\]
\end{lemma}
\begin{proof}
\begin{eqnarray*}
\sum_{t=1}^{\infty} \gamma^{-t} \log(t)&\leq& \sum_{t=1}^{\infty} e^{-(1-\gamma)t} \log(t) \\
&=& \sum_{t=2}^{\infty} e^{-(1-\gamma)t} \log(t) \\
&\overset{*}{\leq}& \int_{1}^{\infty} e^{-(1-\gamma)x} \log(x+1)dx \\
&=& \frac{1}{1-\gamma}\int_{1}^{\infty} e^{-u} \log\left(\frac{u}{1-\gamma}+1\right)du \\
& \leq & \frac{1}{1-\gamma} \left( \left[ 1+ \log\left(\frac{1}{1-\gamma}\right)\right] \int_{1}^{\infty} e^{-u} du + \int_{1}^{\infty} e^{-u}\log(u) du \right) \\
&=& \frac{1}{1-\gamma} \left( \left[ 1+ \log\left(\frac{1}{1-\gamma}\right)\right](1/e)+ \int_{1}^{\infty} e^{-u}\log(u) du \right) \\
& \leq & \frac{1}{1-\gamma} \left[ 1+ \log\left(\frac{1}{1-\gamma}\right)\right]
\end{eqnarray*}
where the last step uses a numerical approximation to the indefinite integral
\[
\int_{1}^{\infty} e^{-u}\log(u) du \approx .22
\]
along with the fact that $1/e + .22 \approx .57< 1.$

The inequality (*) uses that for any $t\geq 2$
\[
e^{-(1-\gamma)t} \log(t) \leq \intop_{t-1}^{t} e^{-(1-\gamma)x} \log(x+1)
\]
since $e^{-(1-\gamma)x}$ is decreasing in $x$ and $\log(x)$ is increasing in $x$.
\end{proof}
Finally we can conclude with the proof of Theorem \ref{thm:noisy regret bound}. We have
\begin{eqnarray*}
\boldGamma\left(\tilde{A}, \{A_t: t \in \N\}\right) &=& (1-\alpha^2) \sum_{t=0}^\infty \alpha^{2t} \E\left[\frac{\E_t[\tilde{R} - R_{t,A_t}]^2}{\I_t( \tilde{A}; Y_{t,A_t} | \xi )}\right] \\
&\leq & 2/\delta + 2(1-\alpha^2) \sum_{t=0}^\infty \alpha^{2t} \E[ |\A _{t}|]
\end{eqnarray*}
Since
 \begin{eqnarray*}
 (1-\alpha^2) \sum_{t=0}^\infty \alpha^{2t} \E[ |\A _{t}|] &\leq & (1-\alpha^2) \sum_{t=1}^\infty \alpha^{2t} \left(2+ \log(t)/\delta\right) \\
 & \leq & 3 + (1/\delta)(1-\alpha^2) \sum_{t=1}^\infty \alpha^{2t} \log(t) \\
& \leq  & 3 + (1/\delta)\left[ 1+ \log\left(\frac{1}{1-\alpha^2}\right)\right],
 \end{eqnarray*}
this implies
\[
\boldGamma\left(\tilde{A}, \{A_t: t \in \N\}\right) \leq 6 + 4/\delta + (2/\delta)\log\left(\frac{1}{1-\alpha^2} \right) = O\left( (1/\delta) \log\left(\frac{1}{1-\alpha^2}\right)  \right)
\]
and concludes the proof of Theorem \ref{thm:noisy regret bound}.

\end{document}